\documentclass[11pt]{article}

\usepackage[margin=1in]{geometry}
\usepackage{amsmath,amssymb,amsthm,mathtools}
\usepackage{microtype}
\usepackage{enumitem}
\usepackage[hidelinks]{hyperref}

\newtheorem{theorem}{Theorem}[section]
\newtheorem{proposition}[theorem]{Proposition}
\newtheorem{lemma}[theorem]{Lemma}
\newtheorem{corollary}[theorem]{Corollary}
\theoremstyle{definition}
\newtheorem{definition}[theorem]{Definition}
\theoremstyle{remark}
\newtheorem{remark}[theorem]{Remark}

\newcommand{\Hcal}{\mathcal H}
\newcommand{\Xcal}{\mathcal X}
\newcommand{\Ycal}{\mathcal Y}
\newcommand{\Tcal}{\mathcal T}

\newcommand{\Prb}{\mathbb P}
\newcommand{\E}{\mathbb E}
\newcommand{\ind}{\mathbf 1}
\newcommand{\DSdim}{\operatorname{DSdim}}
\newcommand{\head}{\operatorname{head}}
\newcommand{\im}{\operatorname{im}}
\newcommand{\Unif}{\operatorname{Unif}}
\newcommand{\argmin}{\operatorname*{arg\,min}}

\title{Local Regularization Does Not Characterize\\Multiclass PAC Learnability}
\author{Eric Hou}
\date{July 24, 2026}

\begin{document}
\maketitle

\begin{abstract}
Local regularization assigns each hypothesis a test-point-dependent score and predicts with a minimum-score hypothesis consistent with the sample. Asilis et al. asked whether this principle characterizes multiclass PAC learnability. We give a negative answer. There is a countable class of Daniely--Shalev-Shwartz dimension at most two with realizable PAC sample complexity
\[
O\!\left(\frac{1}{\varepsilon}\log\frac{1}{\delta}\right),
\]
that no local regularizer learns. Hypotheses are edges of complete graphs and instances are tournaments. At a test tournament, the scores fix an edge ranking while the training sample independently removes competitors. Cyclic triangles force enough inversions that surviving competitors produce constant population error at arbitrarily large sample sizes.
\end{abstract}

\section{Introduction}

A central question in statistical learning theory is which simple algorithmic principles characterize learnability. In binary classification, every consistent empirical risk minimizer learns any PAC-learnable class with nearly optimal sample complexity. The multiclass setting is less rigid. Proper learning can fail even for learnable classes, and general optimal learners use sample-dependent one-inclusion structures rather than a fixed ordering of hypotheses \cite{brukhim2022,daniely2014}.

Asilis et al. proposed local regularization as a possible replacement for empirical risk minimization \cite{asilis-open,asilis-regularization}. A local regularizer is a fixed score $\psi(h,x)$ depending on a hypothesis $h$ and a test point $x$. Given a labeled sample, an induced learner may choose, separately at each test point, the prediction of any minimum-score hypothesis in the realizable version space. The resulting predictor may be improper, but the scores are fixed before the sample is observed. Asilis et al. asked whether every PAC-learnable multiclass class is learned by some local regularizer.

Jafar, Asilis, and Dughmi gave a negative answer in the transductive model \cite{jafar2025}. Their construction forces a mistake at a withheld point through a preference cycle. A PAC lower bound needs more. The withheld point may carry negligible population mass, and repeated independent samples may reveal a distinguished hidden point.

Our counterexample uses distributions over tournaments. At scale $q$, the hypotheses are the edges of $K_q$, and an instance is a tournament on its vertices. An edge hypothesis returns the head of that edge. For a target incidence $(e,v)$, the hard distribution is uniform over tournaments that orient $e$ toward $v$, so every observed label is $v$. The sample filters the competing edges incident to $v$; an independent tournament then tests the pointwise ranking fixed by the regularizer. Error is spread across many instances rather than concentrated at one distinguished point.

\begin{theorem}[Main theorem]\label{thm:main}
There exists a countable multiclass hypothesis class $\Hcal$ satisfying
\[
m_{\Hcal}(\varepsilon,\delta)
\le
\left\lceil \frac{\log(2/\delta)}{\varepsilon}\right\rceil
\qquad
\text{for all }\varepsilon,\delta\in(0,1),
\]
and $\DSdim(\Hcal)\le 2$, such that no local regularizer learns $\Hcal$.

More quantitatively, for every local regularizer $\psi$ there is a deterministic learner $A_\psi$ induced by $\psi$ such that, for
\[
\varepsilon_0=\frac{1}{96e},
\qquad
\delta_0=\frac{1}{192e},
\]
there are arbitrarily large sample sizes $n$ and $\Hcal$-realizable distributions $D$ for which
\[
\Prb_{S\sim D^n}
\bigl[L_D(A_\psi(S))>\varepsilon_0\bigr]
>
\delta_0.
\]
\end{theorem}

An explicit learner identifies the target after seeing both endpoint labels, while the two-label image structure gives the DS-dimension bound. For the lower bound, fixed tie-breaking turns the local scores into a strict edge order at each tournament. The hard sample independently removes competitors, and cyclic triangles leave enough inversions to give constant expected risk for some target incidence.

\section{Preliminaries}

Let $\Xcal$ be an instance space, $\Ycal$ a label space, and $\Hcal\subseteq \Ycal^{\Xcal}$ a hypothesis class. For a finite labeled sample
\[
S=((x_1,y_1),\ldots,(x_n,y_n)),
\]
define its realizable version space by
\[
V_{\Hcal}(S)
=
\{h\in\Hcal:h(x_i)=y_i\text{ for every }i\in[n]\}.
\]
The sample is realizable when $V_{\Hcal}(S)\neq\varnothing$. For a distribution $D$ on $\Xcal\times\Ycal$ and a predictor $g:\Xcal\to\Ycal$, let
\[
L_D(g)=\Prb_{(x,y)\sim D}[g(x)\neq y].
\]
The distribution $D$ is $\Hcal$-realizable if some $h^\star\in\Hcal$ has zero $D$-risk.

A learner maps every finite sample to a predictor. It PAC learns $\Hcal$ if there is a function $m_A:(0,1)^2\to\mathbb N$ such that, for every $\Hcal$-realizable distribution $D$, every $\varepsilon,\delta\in(0,1)$, and every $n\ge m_A(\varepsilon,\delta)$,
\[
\Prb_{S\sim D^n}[L_D(A(S))>\varepsilon]\le\delta.
\]
The unrestricted sample complexity $m_{\Hcal}$ is the pointwise infimum of $m_A$ over all learners $A$.

\begin{definition}[Local regularization]\label{def:local-reg}
A local regularizer is a function
\[
\psi:\Hcal\times\Xcal\to\mathbb R_{\ge0}.
\]
A learner $A$ is induced by $\psi$ if, for every nonempty realizable sample $S$ and every $x\in\Xcal$,
\[
A(S)(x)
\in
\left\{
 h(x):
 h\in\argmin_{g\in V_{\Hcal}(S)}\psi(g,x)
\right\}.
\]
The behavior of $A$ on empty or unrealizable samples is arbitrary. We say that $\psi$ learns $\Hcal$ if every learner induced by $\psi$ PAC learns $\Hcal$.
\end{definition}

\begin{remark}[Nonvacuity and attainment]\label{rem:nonvacuity}
The induction condition is restricted to nonempty realizable samples because $V_{\Hcal}(S)$ is empty on an unrealizable sample. There is a second issue on infinite version spaces: the score infimum need not be attained, in which case the argmin is empty and the universal statement over induced learners becomes vacuous.

For the class constructed below, every nonempty realizable version space is finite. Each local regularizer therefore attains its minimum wherever the induction condition is used, so no separate well-definedness assumption is required.
\end{remark}

\section{The tournament class}

For an integer $q\ge3$, let
\[
V_q=[q],
\qquad
E_q=\binom{V_q}{2},
\qquad
M_q=|E_q|=\binom q2.
\]
Let $\Tcal_q$ be the set of all tournaments on $V_q$. For $\tau\in\Tcal_q$ and $e=\{u,v\}\in E_q$, write $\head_\tau(e)\in e$ for the head of the directed edge $e$ in $\tau$.

Define the countable instance and label spaces
\[
\Xcal
=
\bigsqcup_{q\ge3}(\{q\}\times\Tcal_q),
\qquad
\Ycal
=
\bigsqcup_{q\ge3}(\{q\}\times V_q).
\]
For every $q\ge3$ and edge $e=\{u,v\}\in E_q$, define $h_{q,e}:\Xcal\to\Ycal$ by
\begin{equation}\label{eq:hypothesis}
h_{q,e}(s,\tau)
=
\begin{cases}
(q,\head_\tau(e)), & s=q,\\
(q,\min e), & s\neq q.
\end{cases}
\end{equation}
Here, in the second line, the tournament component belongs to $\Tcal_s$ and is ignored. The constant off-block value makes $h_{q,e}$ a total function without mixing blocks; since every output retains the first coordinate $q$, one observed label already identifies the block of any consistent hypothesis. Let
\[
\Hcal=\{h_{q,e}:q\ge3,\ e\in E_q\}.
\]
This class is countable. Each hypothesis uses exactly two labels:
\begin{equation}\label{eq:image}
\im(h_{q,\{u,v\}})=\{(q,u),(q,v)\}.
\end{equation}
Moreover, distinct hypotheses have distinct image sets.

\subsection{An explicit PAC learner}

\begin{proposition}\label{prop:pac}
For every $\varepsilon,\delta\in(0,1)$,
\[
m_{\Hcal}(\varepsilon,\delta)
\le
\left\lceil\frac{\log(2/\delta)}{\varepsilon}\right\rceil.
\]
\end{proposition}

\begin{proof}
On a nonempty realizable sample, let $A^\star$ output $h_{q,\{u,v\}}$ if two distinct labels $(q,u)$ and $(q,v)$ appear, and the constant predictor $x\mapsto y$ if only one label $y$ appears. Its behavior on empty or unrealizable samples is arbitrary.

Fix a realizable distribution with target $h_{q,\{u,v\}}$, and let $\alpha$ be the probability of label $(q,u)$. Once both labels appear, Equation~\eqref{eq:image} identifies the target and the risk is zero. If only $(q,u)$ appears, the event has probability $\alpha^n$ and the risk is $1-\alpha$; if only $(q,v)$ appears, the corresponding quantities are $(1-\alpha)^n$ and $\alpha$. Hence
\begin{align*}
\Prb\bigl[L_D(A^\star(S))>\varepsilon\bigr]
&\le
\ind\{1-\alpha>\varepsilon\}\alpha^n
+
\ind\{\alpha>\varepsilon\}(1-\alpha)^n\\
&\le 2e^{-\varepsilon n}.
\end{align*}
This is at most $\delta$ at the threshold stated in the proposition.
\end{proof}

\subsection{A dimension bound}

We use the standard Daniely--Shalev-Shwartz dimension. A finite nonempty set $P\subseteq\Ycal^d$ is a $d$-dimensional pseudo-cube if every $z\in P$ has, for every coordinate $i\in[d]$, a point $z'\in P$ that differs from $z$ at coordinate $i$ and agrees at all other coordinates. A $d$-tuple of instances is DS-shattered when the corresponding restriction of the class contains a $d$-dimensional pseudo-cube.

The image condition rules out a three-dimensional pseudo-cube: the required coordinate neighbors would force two distinct hypotheses to have the same two-label image.

\begin{lemma}\label{lem:ds}
Suppose every hypothesis in a class $\mathcal G$ uses at most two labels and the map $g\mapsto\im(g)$ is injective. Then $\DSdim(\mathcal G)\le2$.
\end{lemma}

\begin{proof}
It suffices to exclude a three-dimensional pseudo-cube, since projection onto any subset of coordinates preserves the pseudo-cube property. Suppose that $P\subseteq\Ycal^3$ is such a pseudo-cube in a restriction of $\mathcal G$. Every vector in $P$ uses at most two labels. There is a vector using exactly two: if a chosen vector is constant, any coordinate-neighbor differs in one coordinate and therefore uses two labels. After permuting coordinates, write this vector as $(a,a,b)$ with $a\neq b$.

Its neighbor in the first coordinate has the form $(c,a,b)$ with $c\neq a$. It already contains $a$ and $b$ and uses at most two labels, so $c=b$. Let $g,g'\in\mathcal G$ realize $(a,a,b)$ and $(b,a,b)$. These hypotheses are distinct, but their global image sets both contain $a,b$ and have size at most two. Thus
\[
\im(g)=\im(g')=\{a,b\},
\]
contrary to injectivity.
\end{proof}

\begin{corollary}\label{cor:ds}
The tournament class satisfies $\DSdim(\Hcal)\le2$.
\end{corollary}

\begin{proof}
Apply Lemma~\ref{lem:ds} using Equation~\eqref{eq:image}.
\end{proof}

\section{No local regularizer learns the class}

Because $\Hcal$ is countable, fix an enumeration
\[
\Hcal=\{g_1,g_2,\ldots\}.
\]
Fix an arbitrary local regularizer $\psi$. At each test point $x$, refine its scores to a strict total order by declaring
\begin{equation}\label{eq:order}
g_i\prec_x g_j
\quad\Longleftrightarrow\quad
(\psi(g_i,x),i)<_{\mathrm{lex}}(\psi(g_j,x),j).
\end{equation}
For a nonempty realizable sample $S$, the first coordinate of any observed label fixes one block, so $V_{\Hcal}(S)$ is finite. It therefore has a unique $\prec_x$-least element. Let $A_\psi(S)(x)$ be the label of this element, and define $A_\psi$ arbitrarily on empty or unrealizable samples. The chosen hypothesis minimizes $\psi(\cdot,x)$ on the version space, hence $A_\psi$ is induced by $\psi$.

A regularizer learns $\Hcal$ only if every induced learner is PAC. It is therefore enough to show that $A_\psi$ is not PAC.

\subsection{Hard constant-label distributions}

Fix $q\ge3$, an edge $e\in E_q$, and an endpoint $v\in e$. Define $D_{e,v}$ to be the uniform distribution on
\begin{equation}\label{eq:hard-distribution}
\left\{
\bigl((q,\tau),(q,v)\bigr):
\tau\in\Tcal_q,\ \head_\tau(e)=v
\right\}.
\end{equation}
The distribution is realizable by $h_{q,e}$, and all labels are equal to $(q,v)$.

\begin{lemma}[Version-space structure]\label{lem:version-space}
Let $S\sim D_{e,v}^n$ with $n\ge1$. A hypothesis is consistent with $S$ if and only if it is $h_{q,f}$ for an edge $f$ incident to $v$ and every sampled tournament orients $f$ toward $v$. In particular, $h_{q,e}$ is always consistent. For distinct competitors $f\neq e$ incident to $v$, the consistency events are mutually independent, and each has probability
\[
p_n=2^{-n}.
\]
\end{lemma}

\begin{proof}
The label $(q,v)$ forces a consistent hypothesis to lie in block $q$ and to correspond to an edge incident to $v$. For such an edge $f$, Equation~\eqref{eq:hypothesis} gives the label $(q,v)$ on a sampled tournament exactly when $f$ is oriented toward $v$. This condition holds automatically for $f=e$ by Equation~\eqref{eq:hard-distribution}.

The unordered edges of a uniform tournament have independent unbiased orientations. Conditioning on the orientation of $e$ leaves all other orientations independent and unbiased, and the sampled tournaments are independent. Each non-target competitor therefore survives with probability $2^{-n}$, with mutual independence across competitors.
\end{proof}

\subsection{Risk as a weighted inversion count}

Write
\[
d=q-1.
\]
For every incidence $(e,v)$ and every test tournament $\tau\in\Tcal_q$, define
\begin{equation}\label{eq:rho-global}
\rho_{e,v}^{(n)}(\tau)
=
\Prb_{S\sim D_{e,v}^n}
\bigl[A_\psi(S)(q,\tau)\neq(q,v)\bigr].
\end{equation}
We use this definition for every $\tau$, although the formula below concerns the case in which $e$ is incoming at $v$.

Fix a test tournament $\tau\in\Tcal_q$ and a vertex $v\in V_q$. Restrict the total order in Equation~\eqref{eq:order} at $(q,\tau)$ to the $d$ edges incident to $v$, and list them as
\[
e_1^v\prec_{(q,\tau)}e_2^v
\prec_{(q,\tau)}\cdots
\prec_{(q,\tau)}e_d^v,
\]
where we identify an edge with its corresponding block-$q$ hypothesis. Define
\[
b_j^v(\tau)
=
\begin{cases}
0, & \head_\tau(e_j^v)=v,\\
1, & \head_\tau(e_j^v)\neq v.
\end{cases}
\]
Thus $b_j^v=0$ denotes an incoming edge at $v$, and $b_j^v=1$ an outgoing edge.

When the target is $e_r^v$, write
\[
\rho_{v,r}^{(n)}(\tau)
=
\rho_{e_r^v,v}^{(n)}(\tau).
\]

For a small example, suppose four incident edges are ordered as outgoing, incoming, outgoing, incoming, with the last edge as the target. The target is always present. If the first edge survives, it is selected and the prediction is wrong; if it is removed but the second edge survives, the prediction is correct; if the first two are removed and the third survives, the prediction is wrong. In general, an error occurs exactly when the first surviving edge before an incoming target is outgoing.

\begin{lemma}[Weighted inversions]\label{lem:weighted-inversions}
If $b_r^v(\tau)=0$, then
\begin{equation}\label{eq:weighted-exact}
\rho_{v,r}^{(n)}(\tau)
=
\sum_{\substack{j<r\\b_j^v(\tau)=1}}
p_n(1-p_n)^{j-1}.
\end{equation}
Consequently, if
\[
I_v(\tau)
=
\left|
\left\{
(j,r):j<r,\ b_j^v(\tau)=1,\ b_r^v(\tau)=0
\right\}
\right|,
\]
then
\begin{equation}\label{eq:weighted-lower}
\sum_{r:b_r^v(\tau)=0}
\rho_{v,r}^{(n)}(\tau)
\ge
p_n(1-p_n)^{d-1}I_v(\tau).
\end{equation}
\end{lemma}

\begin{proof}
The target $e_r^v$ always survives, while every other incident edge survives independently with probability $p_n$. If $j<r$, all edges in positions $1,\ldots,j$ are non-target competitors. The learner selects the outgoing edge in position $j$ exactly when that edge survives and the preceding $j-1$ edges do not, an event of probability
\[
p_n(1-p_n)^{j-1}.
\]
These selection events are disjoint, and an error occurs precisely when the first surviving edge before the target is outgoing. This proves Equation~\eqref{eq:weighted-exact}.

Summing over incoming target positions gives
\[
\sum_{r:b_r^v=0}\rho_{v,r}^{(n)}(\tau)
=
\sum_{j:b_j^v=1}
 p_n(1-p_n)^{j-1}
 \bigl|\{r>j:b_r^v=0\}\bigr|.
\]
The summands on the left use the hard distributions associated with their respective target incidences; the display is a numerical sum, not a probability under a common training sample. Each weight on the right is at least $p_n(1-p_n)^{d-1}$, and the unweighted sum is $I_v(\tau)$.
\end{proof}

Set
\begin{equation}\label{eq:nq}
n_q=\left\lceil\log_2(q-1)\right\rceil,
\qquad
p_q=2^{-n_q}.
\end{equation}
Then
\begin{equation}\label{eq:pq-range}
\frac{1}{2d}\le p_q\le\frac{1}{d}.
\end{equation}
Since
\[
\left(1-\frac1d\right)^{d-1}\ge e^{-1}
\qquad(d\ge2),
\]
Equation~\eqref{eq:pq-range} yields
\begin{equation}\label{eq:weight-bound}
p_q(1-p_q)^{d-1}
\ge
\frac{1}{2ed}.
\end{equation}
Let
\[
I(\tau)=\sum_{v\in V_q}I_v(\tau)
\]
and
\[
W(\tau)
=
\sum_{v\in V_q}
\sum_{r:b_r^v(\tau)=0}
\rho_{v,r}^{(n_q)}(\tau).
\]
Applying Equation~\eqref{eq:weight-bound} to Equation~\eqref{eq:weighted-lower} gives
\begin{equation}\label{eq:W-lower}
W(\tau)
\ge
\frac{I(\tau)}{2e(q-1)}.
\end{equation}

\subsection{Cyclic triangles force inversions}

Let $T(\tau)$ denote the number of directed cyclic triangles in the tournament $\tau$.

\begin{lemma}[Triangle charging]\label{lem:triangle}
For every tournament $\tau\in\Tcal_q$ and every strict total ordering of its edge hypotheses,
\[
I(\tau)\ge T(\tau).
\]
\end{lemma}

\begin{proof}
Consider a directed cyclic triangle
\[
u\longrightarrow v\longrightarrow w\longrightarrow u
\]
with edges
\[
a=\{u,v\},
\qquad
b=\{v,w\},
\qquad
c=\{w,u\}.
\]
The three comparisons use the same strict total order at $(q,\tau)$. If there were no inversion at any vertex, then the incoming edge would precede the outgoing edge at each vertex, giving
\[
a\prec b\prec c\prec a.
\]
Every cyclic triangle therefore contains an inversion.

An inversion is an ordered pair of distinct incident edges. It determines the common vertex and the two remaining endpoints, hence a unique triangle. Inversions belonging to distinct triangles are therefore disjoint, so $I(\tau)\ge T(\tau)$.
\end{proof}

For a uniformly random tournament, each triple of vertices is cyclic with probability $1/4$. Therefore
\begin{equation}\label{eq:cyclic-expectation}
\E_{\tau\sim\Unif(\Tcal_q)}T(\tau)
=
\frac14\binom q3.
\end{equation}
Lemma~\ref{lem:triangle}, Equation~\eqref{eq:W-lower}, and the expectation in Equation~\eqref{eq:cyclic-expectation} give
\begin{equation}\label{eq:expected-W}
\E_{\tau}W(\tau)
\ge
\frac{1}{8e(q-1)}\binom q3.
\end{equation}

\subsection{Averaging over target incidences}

For an incidence $(e,v)$ with $e\in E_q$ and $v\in e$, let
\[
R_{e,v}^{(q)}
=
\E_{S\sim D_{e,v}^{n_q}}
L_{D_{e,v}}(A_\psi(S)).
\]
Here the risk averages over the training sample and an independent test tournament from the same hard distribution. An unconditioned uniform tournament points a fixed edge toward either endpoint with probability $1/2$, which accounts for the factor of $2$ in the proof below.

\begin{lemma}[Incidence averaging]\label{lem:incidence}
For every $q\ge3$,
\begin{equation}\label{eq:incidence-identity}
\frac{1}{2M_q}
\sum_{e\in E_q}\sum_{v\in e}R_{e,v}^{(q)}
=
\frac{1}{M_q}
\E_{\tau\sim\Unif(\Tcal_q)}W(\tau).
\end{equation}
\end{lemma}

\begin{proof}
For a fixed incidence $(e,v)$, population risk is the error probability on an independent test tournament conditioned on $\head_\tau(e)=v$:
\[
R_{e,v}^{(q)}
=
\E_{\tau\mid\head_\tau(e)=v}
\rho_{e,v}^{(n_q)}(\tau).
\]
A uniform tournament orients $e$ toward either endpoint with probability $1/2$. Thus
\[
R_{e,v}^{(q)}
=
2\E_\tau
\left[
\ind\{\head_\tau(e)=v\}
\rho_{e,v}^{(n_q)}(\tau)
\right].
\]
For fixed $\tau$ and $e$, exactly one endpoint is the head. Summing over all $2M_q$ incidences therefore gives
\[
\sum_{e\in E_q}\sum_{v\in e}R_{e,v}^{(q)}
=
2\E_\tau W(\tau).
\]
Division by $2M_q$ proves Equation~\eqref{eq:incidence-identity}.
\end{proof}

Substituting Equation~\eqref{eq:expected-W} into Equation~\eqref{eq:incidence-identity} yields
\begin{align}
\frac{1}{2M_q}
\sum_{e\in E_q}\sum_{v\in e}R_{e,v}^{(q)}
&\ge
\frac{1}{8e(q-1)M_q}\binom q3\notag\\
&=
\frac{q-2}{24e(q-1)}
\ge
\frac{1}{48e}.
\label{eq:average-risk}
\end{align}
Hence, for every $q\ge3$, some incidence $(e_q,v_q)$ satisfies
\begin{equation}\label{eq:chosen-incidence}
R_{e_q,v_q}^{(q)}\ge c_0,
\qquad
c_0=\frac{1}{48e}.
\end{equation}

\subsection{Failure of PAC learning}

\begin{theorem}\label{thm:no-local}
No local regularizer learns $\Hcal$.
\end{theorem}

\begin{proof}
Fix $\psi$ and the induced learner $A_\psi$ constructed from Equation~\eqref{eq:order}. For each $q$, choose $(e_q,v_q)$ as in Equation~\eqref{eq:chosen-incidence}, and let
\[
Z_q
=
L_{D_{e_q,v_q}}(A_\psi(S)),
\qquad
S\sim D_{e_q,v_q}^{n_q}.
\]
Then $0\le Z_q\le1$ and $\E Z_q\ge c_0$. For every $\eta\in(0,1)$,
\[
\E Z_q
\le
\eta+\Prb[Z_q>\eta].
\]
Taking $\eta=c_0/2$ gives
\[
\Prb\left[Z_q>\frac{c_0}{2}\right]
\ge
\frac{c_0}{2}.
\]
Set
\[
\varepsilon_0=\frac{c_0}{2}=\frac{1}{96e},
\qquad
\delta_0=\frac{c_0}{4}=\frac{1}{192e}.
\]
Then, for every $q$,
\[
\Prb_{S\sim D_{e_q,v_q}^{n_q}}
\left[
L_{D_{e_q,v_q}}(A_\psi(S))>\varepsilon_0
\right]
\ge
2\delta_0
>
\delta_0.
\]
The sample sizes
\[
n_q=\left\lceil\log_2(q-1)\right\rceil
\]
are unbounded. Given a proposed PAC threshold $m$, choose $q\ge2^m+1$. Then $n_q\ge m$, and the realizable distribution $D_{e_q,v_q}$ violates the $(\varepsilon_0,\delta_0)$ guarantee at sample size $n_q$. Thus $A_\psi$ is not PAC. Since it is induced by the arbitrary regularizer $\psi$, that regularizer does not learn $\Hcal$.
\end{proof}

\begin{proof}[Proof of Theorem~\ref{thm:main}]
The sample-complexity upper bound is Proposition~\ref{prop:pac}, the DS-dimension bound is Corollary~\ref{cor:ds}, and the impossibility for local regularization is Theorem~\ref{thm:no-local}.
\end{proof}

\section{Discussion}

The learner in Proposition~\ref{prop:pac} succeeds after both endpoint labels of the target edge appear. A local regularizer has less freedom: at each tournament it fixes a ranking of the edge hypotheses before seeing the sample. Training data may remove edges from that ranking, but it cannot reorder those that remain.

This differs from the transductive lower bound of Jafar, Asilis, and Dughmi \cite{jafar2025}, which forces an error at a single withheld point. One such point may have negligible mass in a PAC distribution, and repeated sampling may reveal it. The present construction instead uses a constant-label distribution supported on exponentially many tournaments. Sampling filters many competitors, while population risk is evaluated on a fresh tournament from the same distribution.

The theorem applies to local scores fixed independently of the sample. It does not address the unsupervised local-regularization framework of Asilis et al. \cite{asilis-regularization}, where the rule may depend on the unlabeled sample.

\section{Conclusion}

The tournament class is realizably PAC learnable with sample complexity
\[
O\!\left(\frac{1}{\varepsilon}\log\frac{1}{\delta}\right)
\]
and has Daniely--Shalev-Shwartz dimension at most two. Nevertheless, every local regularizer admits an induced learner that fails at arbitrarily large sample sizes. Under the nonvacuous formulation used here, local regularization therefore does not characterize realizable multiclass PAC learnability.


\begin{thebibliography}{9}

\bibitem{asilis-open}
Julian Asilis, Siddartha Devic, Shaddin Dughmi, Vatsal Sharan, and Shang-Hua Teng.
\newblock Open problem: Can local regularization learn all multiclass problems?
\newblock In \emph{Proceedings of the 37th Conference on Learning Theory}, volume 247 of \emph{Proceedings of Machine Learning Research}, pages 5301--5305, 2024.

\bibitem{asilis-regularization}
Julian Asilis, Siddartha Devic, Shaddin Dughmi, Vatsal Sharan, and Shang-Hua Teng.
\newblock Regularization and optimal multiclass learning.
\newblock In \emph{Proceedings of the 37th Conference on Learning Theory}, volume 247 of \emph{Proceedings of Machine Learning Research}, pages 260--310, 2024.

\bibitem{brukhim2022}
Nataly Brukhim, Daniel Carmon, Irit Dinur, Shay Moran, and Amir Yehudayoff.
\newblock A characterization of multiclass learnability.
\newblock In \emph{63rd Annual IEEE Symposium on Foundations of Computer Science}, pages 943--955. IEEE, 2022.

\bibitem{daniely2014}
Amit Daniely and Shai Shalev-Shwartz.
\newblock Optimal learners for multiclass problems.
\newblock In \emph{Proceedings of the 27th Conference on Learning Theory}, volume 35 of \emph{Proceedings of Machine Learning Research}, pages 287--316, 2014.

\bibitem{jafar2025}
Sky Jafar, Julian Asilis, and Shaddin Dughmi.
\newblock Local regularizers are not transductive learners.
\newblock In \emph{Proceedings of the 38th Conference on Learning Theory}, volume 291 of \emph{Proceedings of Machine Learning Research}, pages 2942--2957, 2025.

\end{thebibliography}
\end{document}